\pdfoutput=1
\documentclass{article}

\usepackage{verbatim}
\usepackage{microtype}
\usepackage{epsfig}
\usepackage{graphicx}
\usepackage{setspace}
\usepackage{subcaption}
\usepackage{booktabs} 
\usepackage{amsmath}
\usepackage{amssymb}
\usepackage{mathtools}
\usepackage{natbib}
\DeclareMathOperator*{\argmax}{arg\,max}
\DeclareMathOperator*{\argmin}{arg\,min}
\DeclareMathOperator*{\diag}{diag}
 
\usepackage{amsthm}

\newtheorem{theorem}{Theorem}
\usepackage{hyperref}

\usepackage[accepted]{icml2020}

\icmltitlerunning{Improving Robustness of Deep-Learning-Based Image Reconstruction}

\begin{document}

\twocolumn[
\icmltitle{Improving Robustness of Deep-Learning-Based Image Reconstruction}




\begin{icmlauthorlist}
\icmlauthor{Ankit Raj}{too}
\icmlauthor{Yoram Bresler}{too}
\icmlauthor{Bo Li}{too}
\end{icmlauthorlist}
\icmlaffiliation{too}{University of Illinois at Urbana-Champaign, USA}
\icmlcorrespondingauthor{Ankit Raj}{ankitr3@illinois.edu}
\icmlcorrespondingauthor{Yoram Bresler}{ybresler@illinois.edu}
\icmlcorrespondingauthor{Bo Li}{lbo@illinois.edu}
\icmlkeywords{Machine Learning, ICML}
\vskip 0.3in
]
\begin{NoHyper}
\printAffiliationsAndNotice{}
\end{NoHyper}

\begin{abstract}
Deep-learning-based methods for different applications have been shown vulnerable to adversarial examples. These examples make deployment of such models in safety-critical tasks questionable. Use of deep neural networks as inverse problem solvers has generated much excitement for medical imaging including CT and MRI, but recently a similar vulnerability has also been demonstrated for these tasks. We show that for such inverse problem solvers, one should analyze and study the effect of adversaries in the measurement-space, instead of the signal-space as in previous work. In this paper, we propose to modify the training strategy of  end-to-end deep-learning-based inverse problem solvers to improve robustness. We introduce an auxiliary network to generate adversarial examples, which is used in a min-max formulation to build robust image reconstruction networks. Theoretically, we show for a linear reconstruction scheme the min-max formulation results in a singular-value(s) filter regularized solution, which suppresses the effect of adversarial examples occurring because of ill-conditioning in the measurement matrix. We find that a linear network using the proposed min-max learning scheme indeed converges to the same solution. In addition, for non-linear Compressed Sensing (CS) reconstruction using deep networks, we show significant improvement in robustness using the proposed approach over other methods. We complement the theory by experiments for CS on two different datasets and evaluate the effect of increasing perturbations on trained networks. We find the behavior for ill-conditioned and well-conditioned measurement matrices to be qualitatively different.
\end{abstract}

\section{Introduction}
\label{intro}
Adversarial examples for  deep learning based methods have been demonstrated for different problems \cite{szegedy2013intriguing, kurakin2016adversarial, cisse2017houdini, eykholt2017robust, xiao2018generating}. It has been shown that with minute perturbations, these networks can be made to produce unexpected results. Unfortunately, these perturbations can be obtained very easily. There has been plethora of work to defend against these attacks as well \cite{madry2017towards, tramer2017ensemble, athalye2018obfuscated, wong2018scaling, Jang_2019_ICCV, jiang2018learning, xu2017feature, schmidt2018adversarially}. Recently, \cite{antun2019instabilities, choi2019evaluating} introduced adversarial attacks on image reconstruction networks. In this work, we propose an adversarial training scheme for image reconstruction deep networks to provide robustness. \\
Image reconstruction involving the recovery of an image from indirect measurements  is used in many applications, including critical applications such as medical imaging, e.g.,  Magnetic Resonance Imaging (MRI), Computerised Tomography (CT) etc.  Such applications demand the reconstruction to be stable and reliable. On the other hand, in order to speed up the acquisition, reduce sensor cost, or reduce radiation dose, it is highly desirable to subsample the measurement data, while still recovering the original image. This is enabled by the compressive sensing (CS) paradigm \cite{candes2006stable, donoho2006compressed}. CS  involves projecting a high dimensional,  signal $x\in\mathbb{R}^n$ to a lower dimensional measurement $y\in\mathbb{R}^m, m\ll n$, using a small set of linear, non-adaptive frames. The noisy measurement model is:
\begin{equation}
    y = Ax+v,A\in\mathbb{R}^{m\times n}, v\sim \mathcal{N}(0,\sigma^2 I)
\end{equation}
where $A$ is the measurement matrix. The goal is to recover the unobserved natural image $x$, from the compressive measurement $y$. Although the problem with $m \ll n$ is severely ill-posed and does not have a unique solution, CS achieves nice, stable solutions for a special class of signals $x$ - those that are sparse or sparsifiable, by using  sparse regularization techniques \cite{candes2006stable, donoho2006compressed, elad2006image, dong2011image, wen2015structured, liu2017image, dabov2009bm3d, yang2010image, elad2010sparse, li2009user, ravishankar2012learning}. \\
Recently, deep learning based methods have also been proposed as an alternative method for performing image reconstruction \cite{zhu2018image, jin2017deep, schlemper2017deep, yang2017dagan, hammernik2018learning}. While these methods have achieved state-of-the-art (SOTA) performance, the networks have been found to be very unstable \cite{antun2019instabilities}, as compared to the traditional methods. Adversarial perturbations have been shown to exist for such networks, which can degrade the quality of image reconstruction significantly. \cite{antun2019instabilities} studies three types of instabilities: \textit{\textbf{(i)}}   Tiny (small norm) perturbations applied to images that are almost invisible in the original images, but cause a significant distortion in the reconstructed images. \textit{\textbf{(ii)}} Small structural changes in the original images, that get removed from the reconstructed images. \textit{\textbf{(iii)}} Stability with increasing the number of measurement samples. We try to address instability (i) above.\\
In this paper, we argue that studying the instability for image reconstruction networks in the $x$-space as addressed by \cite{antun2019instabilities} is sub-optimal and instead, we should consider perturbations in the measurement, $y$-space. To improve robustness, we modify the training strategy:  we introduce an auxiliary network to generate adversarial examples on the fly, which are used in a min-max formulation. This results in an adversarial game between two networks while training, similar to the Generative Adversarial Networks (GANs) \cite{goodfellow2014generative, arjovsky2017wasserstein}. However, since the goal here is to build a robust reconstruction network, we make some changes in the training strategy compared to GANs.\\
Our theoretical analysis for a special case of a linear reconstruction scheme shows that the min-max formulation results in a singular-value filter regularized solution, which suppresses the effect of adversarial examples. 
Our experiment using the min-max formulation with a learned adversarial example generator 
for a linear reconstruction network shows that the network indeed converges to the solution obtained theoretically. For a complex non-linear deep network, our experiments show that training using the proposed formulation results in more robust network, both qualitatively and quantitatively, compared to other methods. Further, we experimented and analyzed the reconstruction for two different measurement matrices, one well-conditioned and another relatively ill-conditioned. We find that the behavior in the two cases is qualitatively different.
\section{Proposed Method}
\subsection{Adversarial Training}
One of the most powerful methods for training an adversarially robust network is adversarial training \cite{madry2017towards, tramer2017ensemble,sinha2017certifying,arnab2018robustness}. It involves training the network using adversarial examples, enhancing the robustness of the network to attacks during inference. This strategy has been quite effective in classification settings, where the goal is to make the network output the correct label corresponding to the adversarial example.\\
Standard adversarial training involves solving the following min-max optimization problem:
\begin{align}
\label{minmax_gen}
    \min_{\theta} \mathbb{E}_{(x, y) \in \mathbb{D}} [\max_{\delta: \|\delta\|_p \leq \epsilon} \mathcal{L}(f(x+\delta; \theta), y)]
\end{align}
where $\mathcal{L}(\cdot)$ represents the applicable loss function, e.g., cross-entropy for classification, and $\delta$ is the perturbation added to each sample, within an $\ell_p$-norm ball of radius $\epsilon$. \\
This min-max formulation encompasses possible variants of adversarial training. It consists of solving two optimization problems: an inner maximization and an outer minimization problem. This corresponds to an adversarial game between the attacker and robust network $f$. The inner problem tries to find the optimal $\delta: \|\delta\|_p \leq \epsilon$ for a given data point $(x,y)$ maximizing the loss, which essentially is the adversarial attack, whereas the outer problem aims to find a $\theta$ minimizing the same loss. For an optimal $\theta^{*}$ solving the 
equation \ref{minmax_gen}, then $f(; \theta^*)$ will be robust (in expected value) to all the $x_{adv}$ lying in the $\epsilon$-radius of $\ell_p$-norm ball around the true $x$.
\subsection{Problem Formulation}
\cite{antun2019instabilities} identify instabilities of a deep learning based image reconstruction network by maximizing the following cost function:
\begin{align} \label{eq:x-space}
    Q_y(r) = \frac{1}{2} \|f(y+Ar) - x\|_2^2 - \frac{\lambda}{2} \|r\|^2
\end{align}
As evident from this framework, the perturbation $r$ is added in the $x$-space for each $y$, resulting in  perturbation $Ar$ in the $y$-space. We argue that this formulation can miss important aspects in image reconstruction, especially in ill-posed problems, for the following three main reasons:
\begin{enumerate}
    \item It may not be able to model all possible perturbations to $y$. The perturbations $A\delta$ to $y$ modeled in this formulation are all constrained to the range-space of $A$. 
    When $A$ does not have full row rank, there exist perturbations to $y$ that cannot be represented as $A\delta$.
    \item 
    It misses instabilities created by the ill-conditioning of the reconstruction problem.
    Consider a simple ill-conditioned reconstruction problem:
    \begin{align}
        A = \begin{bmatrix} 
             1 & 0 \\
             0 & r 
            \end{bmatrix} 
        \text{ and }
            f = \begin{bmatrix} 
             1 & 0 \\
             0 & 1/r 
            \end{bmatrix} 
        \label{ill_cond_2x2}
    \end{align}
    where $A$ and $f$ define the forward and reconstruction operator respectively, and $|r| \ll 1$.  For $\delta=[0, \epsilon]^T$ perturbation in $x$, the  reconstruction is $f(A(x+\delta)) = x+\delta$, and the reconstruction error is $\|f(A(x+\delta)) -x\|_2 = \epsilon$, that is, for small $\epsilon$, the perturbation has negligible effect. In contrast, for the same perturbation $\delta$ in $y$, the reconstruction is $f(Ax+\delta) = x + [0,\epsilon/r]^T$, with reconstruction error $\|f(A(x+\delta)) -x\|_2 = \epsilon/r $, which can be arbitrarily large if $r \rightarrow 0$. This aspect is completely missed by the formulation based on \eqref{eq:x-space}.
    \item For inverse problems, one also wants robustness to perturbations in the measurement matrix $A$. Suppose $A$ used in training is slightly different from the actual $A' = A+ \tilde{A}$ that generates the measurements. This results in  perturbation $\tilde{A}x$ in $y$-space, which may be outside the range space of $A$, and therefore, as in 1 above,  may not be possible to capture by the formulation based on \eqref{eq:x-space}.
\end{enumerate}
The above points indicate that studying the problem of robustness to perturbations for image reconstruction problems in $x$-space misses  possible perturbations in $y$-space that can have a huge adversarial effect on reconstruction. Since many of the image reconstruction problems are ill-posed or ill-conditioned, we formulate and study the issue of adversaries in the $y$-space, which is more generic and able to handle perturbations in the measurement operator $A$ as well.
\subsection{Image Reconstruction}
Image Reconstruction deals with recovering the clean image $x$ from noisy and possibly incomplete measurements $y = Ax+v$. Recently, deep-learning-based approaches have outperformed the traditional  techniques. Many deep learning architectures are inspired by iterative reconstruction schemes  \cite{rick2017one, raj2019gan, bora2017compressed, wen2019transform}. Another  popular way is to use an end-to-end deep network to solve the image reconstruction problem directly \cite{jin2017deep, zhu2018image, schlemper2017deep, yang2017dagan, hammernik2018learning, sajjadi2017enhancenet, yao2019dr2}. In this work, we propose modification in the training scheme for the end-to-end networks. \\
Consider the standard MSE loss in  $x$-space with the popular 
 $\ell_2$-regularization on the weights (aka weight decay), which mitigates overfitting and helps in generalization \cite{krogh1992simple} 
\begin{equation} \label{eq:wd_regul}
    \min_{\theta} \mathbb{E}_{x} \|f(Ax; \theta) - x\|^2 + \mu \|\theta\|^2
\end{equation}
In this paper, we experiment both with $\mu >0$ (regularization present) and $\mu =0$ (no regularization). No regularization is used in the sequel, unless stated otherwise.
\subsubsection{Adversarial Training for Image Reconstruction}
Motivated by the adversarial training strategy \eqref{minmax_gen}, several frameworks have been proposed recently to make classification by deep networks more robust \cite{jang2019adversarial, kurakin2016adversarial, wang2019direct}. For image reconstruction, we propose to modify the training loss to the general form
\begin{equation}
    \label{img_recon_1}
    \min_{\theta} \mathbb{E}_{x} \max_{\delta: \|\delta\|_p \leq \epsilon} \|f(Ax; \theta) - x \|^2 + \lambda \|f(Ax + \delta; \theta) - x\|^2 \nonumber
\end{equation}
The role of the first term is to ensure that the network $f$ maps the non-adversarial measurement to the true $x$, while the role of the second term is to train $f$ on worst-case adversarial examples within the $\ell_p$-norm ball around the nominal measurement $Ax$. We want $\delta$ to be the worst case perturbation for a given $f$. However, during the initial training epochs, $f$ is mostly random (assuming random initialization of the weights) resulting in random perturbation, which makes $f$ diverge. Hence we need only the first term during initial epochs to get a decent $f$ that provides reasonable reconstruction. Then, reasonable perturbations are obtained by activating the second term, which results in robust $f$. \\
Now, solving the min-max problem above is intractable for a large dataset as it involves finding the adversarial example, which requires to solve the inner maximization for each $y = Ax$. This may be done 
using projected gradient descent (PGD), but  is very costly. A possible sub-optimal approximation (with $p = 2$) for this formulation is:
\begin{equation}
    \label{eq:img_recon_2}
    \min_{\theta} \max_{\delta: \|\delta\|_2 \leq \epsilon} \mathbb{E}_{x} \|f(Ax; \theta) - x \|_2^2 + \lambda \|f(Ax + \delta; \theta) - x\|_2^2
\end{equation}
This formulation finds a common $\delta$ which is adversarial to each measurement $y$ and tries to minimize the reconstruction loss for the adversarial examples together with that for clean examples. Clearly this is sub-optimal as using a perturbation $\delta$ common to all $y$'s need not be the worst-case perturbation for any of the $y$'s, and optimizing for the common $\delta$ won't result in a highly robust network. \\
Ideally, we would want the best of both worlds: i.e., to generate $\delta$ for each $y$ independently, together with tractable training. To this end, we propose to parameterize the worst-case perturbation $\delta = \argmax_{\delta: \|\delta\|_2 \leq \epsilon} \|f(y + \delta; \theta) - x\|_2^2$ by a deep neural network $G(y; \phi)$. This also eliminates the need of solving the inner-maximization to find $\delta$ using hand-designed methods. Since $G(\cdot)$ is parameterized by $\phi$ and takes $y$ as input, a well-trained $G$ will result in optimal perturbation for the given $y = Ax$. The modified loss function becomes:
\begin{align*}
    \min_{\theta} \max_{\phi: \|G(\cdot, \phi)\|_2 \leq \epsilon} & \mathbb{E}_{x} \|f(Ax; \theta) - x \|^2 \\
     & + \lambda \|f(Ax + G(Ax; \phi); \theta) - x\|^2 
\end{align*}
This results in an adversarial game between the two networks: $G$ and $f$, where $G$'s goal is to generate strong adversarial examples that maximize the reconstruction loss for the given $f$, while $f$ tries to make itself robust to the adversarial examples generated by the $G$. This framework is illustrated  in the Fig.~\ref{fig:arch}. This min-max setting is quite similar to the Generative adversarial network (GAN), with the difference in the objective function. Also, here, the main goal is to build an adversarially robust $f$, which requires some empirical changes compared to standard GANs to make it work. Another change is to reformulate the constraint $\|G(\cdot, \phi)\|_2 \leq \epsilon$ into a penalty form using the hinge loss, which makes the training more tractable:
\begin{align}
    \min_{\theta} \max_{\phi} \quad & \mathbb{E}_{x} \|f(Ax; \theta) - x \|^2 \nonumber \\
     & + \lambda_1 \|f(Ax + G(Ax; \phi); \theta) - x\|^2   \nonumber \\
     & \quad + \lambda_2  \max \{0, \|G(Ax; \phi)\|_2^2 - \epsilon \} \label{eq: img_recon_3} 
\end{align}
Note that $\lambda_2$ must be negative to satisfy the required constraint $\|G(\cdot, \phi)\|_2 \leq \epsilon$.
\begin{figure}[h]
    \centering
    \includegraphics[width=0.5\textwidth]{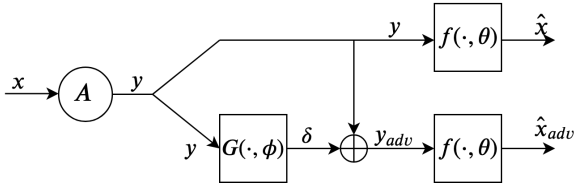}
    \caption{Adversarial training framework of image reconstruction network $f$, jointly with another network $G$, generating the additive perturbations}
    \label{fig:arch}
\end{figure}
\subsubsection{Training Strategy}
We apply some modifications and intuitive changes to train a robust $f$ jointly with training $G$ in a mini-batch set-up. At each iteration, we update $G$ to generate adversarial examples and train  $f$ using those adversarial examples along with the non-adversarial or clean samples to make it robust. Along with the training of robust $f$, $G$ is being trained to generate worst-case adversarial examples. To generate strong adversarial examples by $G$ in the mini-batch update, we divide each mini-batch into $K$ sets. Now, $G$ is trained over each set independently and we use adversarial examples after the update of $G$ for each set. This fine-tunes $G$ for the small set to generate stronger perturbations for every image belonging to the set. Then, $f$ is trained using the entire mini-batch at once but with the adversarial examples generated set-wise. $G$ obtained after the update corresponding to the $K^{th}$ set is passed for the next iteration or mini-batch update. This is described in Algorithm \ref{alg:1}.
\begin{algorithm}[t]
\caption{Algorithm for training at iteration $T$}
\label{alg:1}
\textbf{Input}: Mini-batch samples $(x_T, y_T)$, $G_{T-1}$, $f_{T-1}$ \\
\textbf{Output}: $G_T$ and $f_T$
\begin{algorithmic}[1] 
\STATE $G_{T, 0} = G_{T-1}$, $f = f_{T-1}$ Divide mini-batch into $K$ parts.
\WHILE{$k \leq K$}
\STATE $x = x_{T, k}, G = G_{T, k-1}$
\STATE $G_{T, k} = \argmax_{G} \lambda_1 \|f_{T-1}(Ax + G(Ax; \phi); \theta) -  x\|^2 + \lambda_2  \max \{0, \|G(Ax; \phi)\|_2^2 - \epsilon \}$
\STATE $\delta_{T, k} = G_{T, k}(x)$
\ENDWHILE
\STATE $\delta_T = [\delta_{T,1}, \delta_{T,2}, ..., \delta_{T,K}]$
\STATE $f_{T} = \argmin_{f} \|f(Ax_T) - x_T\|^2 + 
       \lambda_1 \|f(Ax_T + \delta_T) - x_T\|^2$
\STATE $G_T = G_{T, K}$
\STATE \textbf{return} $G_T, f_T$
\end{algorithmic}
\end{algorithm}
\subsection{Robustness Metric}
\label{sec: eval}
We define a  metric to compare the robustness of different networks. We measure the following quantity for network $f$:
\begin{equation} \label{eq:delta_x0}
    \Delta_{\text{max}}(x_0, \epsilon) = \max_{\|\delta\|_2 \leq \epsilon}\|f(Ax_0 + \delta) - x_0\|^2 
\end{equation}
This determines the reconstruction error due to the worst-case additive perturbation over an $\epsilon$-ball around the nominal measurement $y = Ax_0$ for each image $x_0$. The final robustness metric for $f$  is $\rho(\epsilon) = \mathbb{E}_{x_0}[\Delta_{\text{max}}(x_0, \epsilon)]$, which we estimate by the sample average of $\Delta_{\text{max}}(x_0, \epsilon)$ over a test dataset, 
\begin{equation}
    \hat{\rho}(\epsilon) = \frac{1}{N}\sum_{i=1}^{N} \Delta_{\text{max}}(x_i, \epsilon) \label{eq: rho}
\end{equation}
The smaller $\hat{\rho}$, the more robust the network.\\
We solve the optimization problem in \eqref{eq:delta_x0}  using projected gradient ascent (PGA) with momentum (with parameters selected empirically).
Importantly, unlike training, where computation of $\Delta_{\text{max}}(x_0)$ is required at every epoch, we need to solve \eqref{eq:delta_x0}
only once for every sample $x_i$ in the test set, making this computation feasible during testing.

\section{Theoretical Analysis}
\label{sec: theory_ana}
We theoretically obtained the optimal solution for the min-max formulation in \eqref{eq:img_recon_2} for a simple linear reconstruction. Although this analysis doesn't  extend easily to the non-linear deep learning based reconstruction, it gives some insights for the behavior of the proposed formulation and how it depends on the conditioning of the measurement matrices.
\begin{theorem} \label{theorem}
    Suppose 
    that the reconstruction network $f$ is a one-layer feed-forward network with no non-linearity i.e., $f = B$, where matrix $B$  has SVD: $B = MQP^T$. Denote the SVD of the measurement matrix $A$ by $A=USV^T$, where $S$ is a diagonal matrix with singular values in permuted \emph(increasing) order, and assume that the data is normalized, i.e., $E(x) = 0$ and $cov(x) = I$. Then the optimal $B$ obtained by solving \eqref{eq:img_recon_2} is a modified pseudo-inverse of $A$, with $M = V$, $P = U$ and $Q$ a filtered inverse of $S$, given by the diagonal matrix
    \begin{align} \label{eq:theor_res}
       Q & = \diag \left ( q_m, \ldots , q_m , 1/S_{m+1} , \ldots , 1/S_{n} \right), \nonumber \\ 
       q_m &= \frac{\sum_{i=1}^m S_i}{ \sum_{i=1}^m S_i^2 +  \frac{\lambda}{1+\lambda}\epsilon^2}
    \end{align}
 with largest entry $q_m$ of multiplicity $m$ that depends on $\epsilon$, $\lambda$ and $\{S_i\}_{i=1}^n$.
 \end{theorem}
\begin{proof}
    Please refer to the appendix \ref{appen} for the proof.
\end{proof}
The modified inverse $B$ reduces the effect of ill-conditioning in $A$ for adversarial cases in the reconstruction. This can be easily understood, using the simple example from the equation \ref{ill_cond_2x2}. As explained previously, for the $A$ in \eqref{ill_cond_2x2} with $|r| <1$, an exact inverse, $f = \begin{bmatrix}
1 & 0 \\
0 & \frac{1}{r}
\end{bmatrix}$, amplifies the perturbation.  Instead the min-max formulation \eqref{eq:img_recon_2} (with $\lambda = 1$) results in a modified pseudo inverse $\hat{f} = \begin{bmatrix}
1 & 0 \\
0 & \frac{r}{r^2 + 0.5\epsilon^2}
\end{bmatrix}$, suppressing the effect of an adversarial perturbation $\delta=[0, \epsilon]^T$ in $y$ as $\|f\delta\| \gg \|\hat{f}\delta\|$ for $r \rightarrow 0$ and $\epsilon \nrightarrow 0$. It can also be seen that $\hat{f}$ won't be optimal the for the unperturbed $y$ as it's not actual an inverse and reconstruction loss using $f$ for unperturbed case would be smaller than that for $\hat{f}$. However, for even very small adversaries, $f$ would be much more sensitive than $\hat{f}$. It shows the trade-off between the perturbed and unperturbed case for the reconstruction in the case of ill-conditioned $A$. \\ This trade-off behavior will not manifest for  a well-conditioned, as an ideal linear inverse $f$ for this case won't amplify the small perturbations and a reconstruction obtained using \eqref{eq:img_recon_2} with linear $\hat{f}$ will be very close to $f$  (depending on $\epsilon$): for well-conditioned $A$, $r \nrightarrow 0$. In that case $r^2 \gg 0.5 \epsilon^2$, which reduces $\hat{f}$ to $f$.\\
Our experiments with deep-learning-based non-linear image reconstruction methods for CS using as sensing matrices random rows of a Gaussian matrix (well-conditioned) vs. random rows of a  DCT matrix (relatively ill-conditioned) indeed show the qualitatively different behavior with increasing amount of perturbations.
\section{Experiments}
\textbf{Network Architecture:} For the reconstruction network $f$, we follow the architecture of deep convolutional networks for image reconstruction. They use multiple convolution, deconvolution and ReLU layers, and use batch normalization and dropout for better generalization. As a pre-processing step, which has been found to be  effective  for reconstruction, we apply the transpose (adjoint) of $A$ to the measurement $y$, feeding $A^Ty$ to the network. This transforms the measurement into the image-space, allowing the network to operate purely in image space. \\
\begin{table*}[t]
    \centering 
    \hspace{-2.5em}
        \begin{tabular}{c  c}
        \centering
            \begin{subfigure}[b]{0.5\linewidth}
            \centering
            \includegraphics[width=1\textwidth]{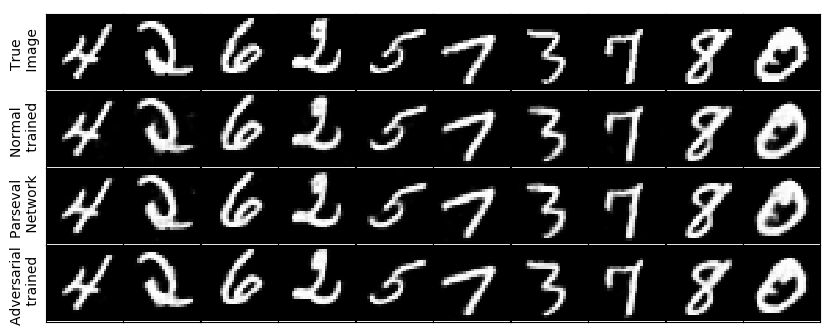}
            \caption{$\epsilon = 0$} \label{mnist_gauss_a}
            \end{subfigure}
            & 
             \begin{subfigure}[b]{0.5\linewidth}
            \centering
            \includegraphics[width=1\textwidth]{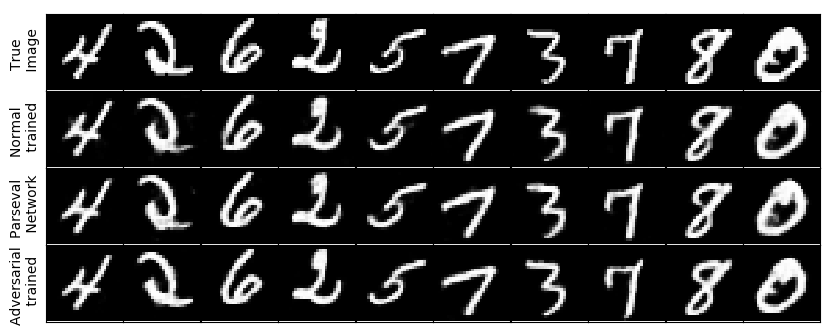}
            \caption{$\epsilon = 1.0$} \label{mnist_gauss_b}
            \end{subfigure}
            \\
            \begin{subfigure}[b]{0.5\linewidth}
            \centering
            \includegraphics[width=1\textwidth]{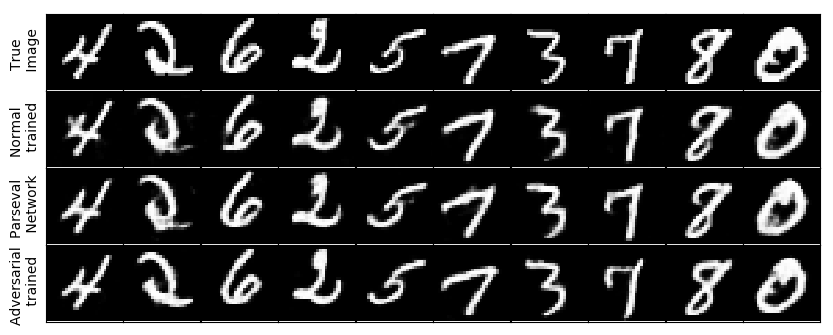}
            \caption{$\epsilon = 2.0$} \label{mnist_gauss_c}
            \end{subfigure} 
            &
            \begin{subfigure}[b]{0.5\linewidth}
            \centering
            \includegraphics[width=1\textwidth]{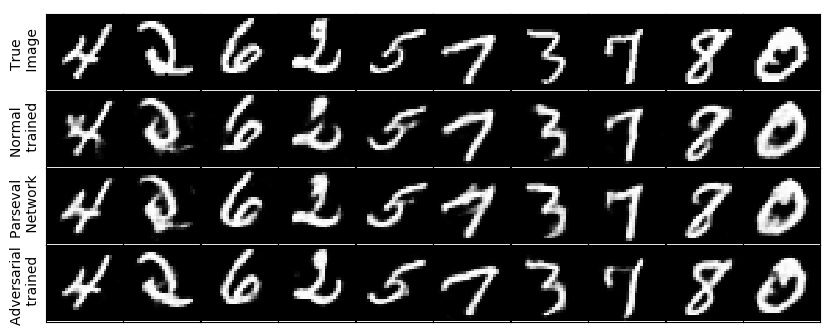}
            \caption{$\epsilon = 3.0$} \label{mnist_gauss_d}
            \end{subfigure}
        \end{tabular} 
    \captionof{figure}{Qualitative Comparison for the MNIST dataset for different perturbations. \emph{First row} of each sub-figure corresponds to the true image, \emph{Second row} to the reconstruction using normally trained model, \emph{Third row} to the reconstruction using Parseval Network, \emph{Fourth row} to the reconstruction using the adversarially trained model (\emph{\textbf{proposed scheme}}).} \label{mnist_gauss_comp}
\end{table*}
For the adversarial perturbation generator $G$ we use a standard feed-forward network, which takes input $y$ as input. The network consists of multiple fully-connected and ReLU layers. We trained the architecture shown in fig. \ref{fig:arch} using the objective defined in the \eqref{eq: img_recon_3}. \\
We designed networks of similar structure but different number of layers for the two datasets, MNIST and CelebA used in the experiments.\\
We used the Adam Optimizer with $\beta_1 = 0.5$, $\beta_2 = 0.999$, learning rate of $10^{-4}$ and mini-batch size of $128$, but divided into $K=4$ parts during the update of $G$, described in the algorithm \ref{alg:1}. During training, the size $\epsilon$ of the perturbation has to be neither too big (affects performance on clean samples) nor too small (results in less robustness). We empirically picked $\epsilon = 2$ for MNIST and $\epsilon = 3$ for the CelebA datasets. However, during testing, we evaluated $\hat{\rho}$, defined in \eqref{eq: rho} for different $\epsilon$'s (including those not used while training), to obtain a fair assessment of robustness.\\
We compare the adversarially trained model using the min-max formulation defined in the objective \ref{eq: img_recon_3}, with three models trained using different training schemes:
\begin{enumerate}
\itemsep0em 
    \item Normally trained model with no regularization, i.e., $\mu=0$ in \eqref{eq: img_recon_3}.
    \item $\ell_2$-norm weight regularized model, using \eqref{eq:wd_regul} with $\mu >10^{-6}$ (aka weight decay), chosen empirically to avoid over-fitting and improve robustness and generalization of the network.
    \item Lipschitz constant ($\mathcal{L}$)-constrained Parseval network \cite{cisse2017parseval}. The idea is to constrain the overall Lipschitz constant $\mathcal{L}$ of the network to be $\leq 1$, by making $\mathcal{L}$ of every layer, $\leq 1$. Motivated by the idea that regularizing the spectral norm of weight matrices could help in the context of robustness, this approach proposes to constrain the weight matrices to also be orthonormal, making them \textit{Parseval tight frames}. Let $S_{fc}$ and $S_{c}$ define the set of indices for fully-connected and convolutional layers respectively. The regularization term to penalize the deviation from the constraint is
    \begin{equation}
        \hspace{-0.8em}\frac{\beta}{2}(\sum_{i \in S_{fc}} \|W_i^TW_i - I_i\|_2^2 + \sum_{j \in S_{c}}\|\mathbf{W_j}^T\mathbf{W_j} - \frac{I_j}{k_j}\|_2^2)
    \end{equation}
    where $W_i$ is the weight matrix for $ith$ fully connected layer and $\mathbf{W_j}$ is the transformed or unfolded weight matrix of $jth$ convolution layer having kernel size $k_j$. This transformation requires input to the convolution to shift and repeat $k_j^2$ times. Hence, to maintain the \textit{Parseval tight frames} constraint on the convolution operator, we need to make $\mathbf{W_j}^T\mathbf{W_j} \approx \frac{I_j}{k_j}$. $I_i$ and $I_j$ are identity matrices whose sizes depend on the size of $W_i$ and $\mathbf{W_j}$ respectively. $\beta$ controls the weight given to the  regularization compared to the standard reconstruction loss. Empirically, we picked $\beta$ to be $10^{-5}$.
\end{enumerate}
To compare different training schemes, we follow the same scheme (described below) for each datasets. Also, we extensively compare the performance for the two datasets for Compressive Sensing (CS) task using  two matrices: one  well-conditioned and another, relatively ill-conditioned. This comparison complements the theoretical analysis, discussed in the previous section.\\
\begin{table*}[t]
    \centering
    \hspace{-2em}
        \begin{tabular}{c c}
            \begin{subfigure}[b]{0.5\linewidth}
            \centering
            \includegraphics[width=1\textwidth]{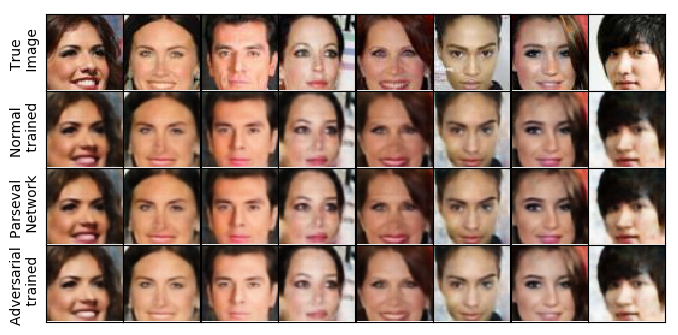}
            \caption{$\epsilon = 0$} \label{celeb_gauss_a}
            \end{subfigure} &
            \begin{subfigure}[b]{0.5\linewidth}
            \centering
            \includegraphics[width=1\textwidth]{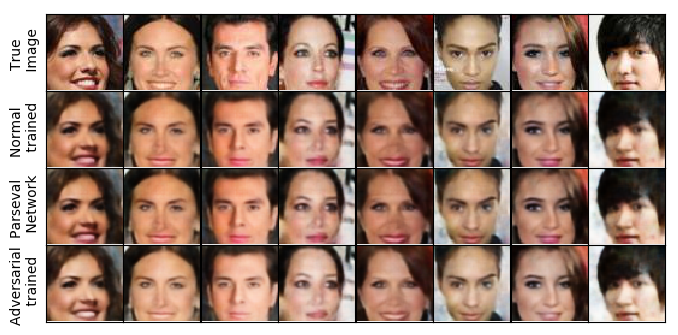}
            \caption{$\epsilon = 2.0$} \label{celeb_gauss_b}
            \end{subfigure} \\
            \begin{subfigure}[b]{0.5\linewidth}
            \centering
            \includegraphics[width=1\textwidth]{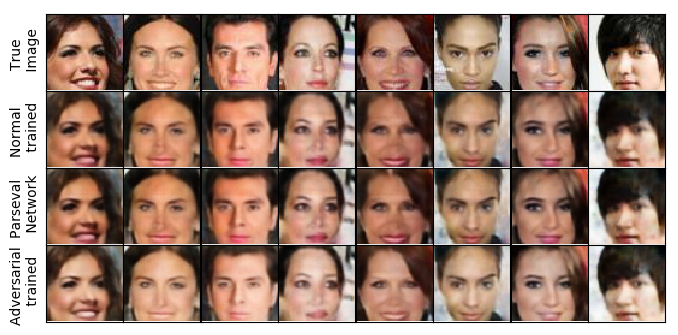}
            \caption{$\epsilon = 5.0$} \label{celeb_gauss_c}
            \end{subfigure} &
            \begin{subfigure}[b]{0.5\linewidth}
            \centering
            \includegraphics[width=1\textwidth]{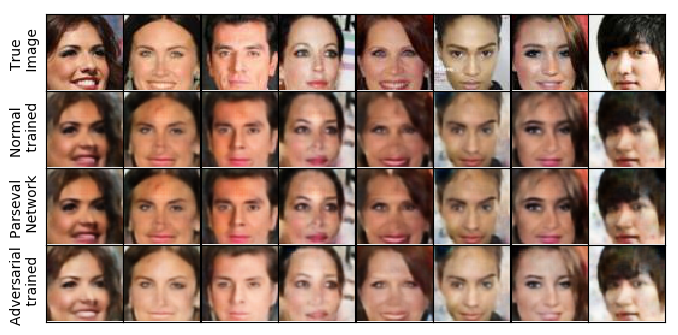}
            \caption{$\epsilon = 10.0$} \label{celeb_gauss_d}
            \end{subfigure}
        \end{tabular}
        \captionof{figure}{Qualitative Comparison for the CelebA dataset for different perturbations. \emph{First row} of each sub-figure corresponds to the true image, \emph{Second row} to the reconstruction using normally trained model, \emph{Third row} to the reconstruction using Parseval Network, \emph{Fourth row} to the reconstruction using the adversarially trained model (\emph{\textbf{proposed scheme}}).} \label{celeb_gauss_comp}
\end{table*}
The \textbf{MNIST} dataset \cite{lecun1998gradient} consists of $28\times28$ gray-scale images of digits with $50,000$ training and $10,000$ test samples. The image reconstruction network consists of $4$ convolution layers and $3$ transposed convolution layers using re-scaled images between $[-1, 1]$. For the generator $G$, we used 5 fully-connected layers network. Empirically, we found  $\lambda_1 = 1$ and $\lambda_2 = -0.1$ in \eqref{eq: img_recon_3}, gave the best performance in terms of robustness (lower $\hat{\rho}$) for different perturbations.\\
The \textbf{CelebA} dataset \cite{liu2015faceattributes} consists of more than $200,000$ celebrity images. We use the aligned and cropped version, which pre-processes each image to a size of $64\times64\times 3$ and scaled between $[-1, 1]$. We randomly pick $160,000$ images for the training. Images from the $40,000$ held-out set are used for evaluation. The image reconstruction network consists of $6$ convolution layers and $4$ transposed convolution layers. For the generator $G$, we used a 6 fully-connected layers network. We found  $\lambda_1 = 3$ and $\lambda_2 = -1$ in \eqref{eq: img_recon_3} gave the best robustness performance (lower $\hat{\rho}$) for different perturbations. 
\subsection{Gaussian Measurement matrix} \label{sec:gauss_mat}
In this set-up, we use the same measurement matrix $A$ as \cite{bora2017compressed, raj2019gan}, i.e. $A_{i, j} \sim N(0, 1/m)$ where $m$ is the number of measurements. For MNIST, the measurement matrix $A \in R^{m\times784}$, with $m = 100$, whereas for CelebA, $A\in R^{m\times12288}$, with $m=1000$. Figures \ref{mnist_gauss_comp} and \ref{celeb_gauss_comp} show the qualitative comparisons for the MNIST and CelebA reconstructions respectively, by solving the optimization described in Section \ref{sec: eval}. It can be seen clearly in both the cases that for different $\epsilon$ the adversarially trained models outperform the normally trained and Parseval networks. For higher $\epsilon$'s, the normally trained and Parseval models generate significant artifacts, which are much less for the adversarially trained models. Figures Fig.~\ref{fig:MNIST gauss} and Fig.~\ref{fig:Celeb gauss} show this improvement in performance in terms of the quantitative metric $\hat{\rho}$, defined in \eqref{eq: rho} for the MNIST and CelebA datasets respectively. It can be seen that $\hat{\rho}$ is lower for the adversarially-trained models compared to other training methods: no regularization, $\ell_2$-norm regularization on weights, and Parseval networks (Lipschitz-constant-regularized) 
for different $\epsilon$'s, showing that adversarial training using the proposed min-max formulation indeed outperforms other approaches in terms of robustness. It is noteworthy that even for $\epsilon = 0$, adversarial training reduces the reconstruction loss, indicating that it acts like an excellent regularizer in general.
\subsection{Discrete Cosine Transform (DCT) matrix}
To empirically study the effect of conditioning of the matrix, we did experiment by choosing $A$ as random $m$ rows and $n$ columns of a $p\times p$ DCT matrix, where $p > n$. This makes $A$  relatively more ill-conditioned than the random Gaussian $A$, i.e. the condition number for the random DCT matrix is higher than that of random Gaussian one. The number of measurements has been kept same as the previous case, i.e. $(m=100,\text{ }n=784)$ for MNIST and $(m=1000,\text{ }n=12288)$ for CelebA. We trained networks having the same configuration as the Gaussian ones. Fig.~\ref{fig:gauss dct} shows the comparison for the two measurement matrices. Based on the figure, we can see that $\hat{\rho}$ for the DCT, MNIST (Fig.~\ref{fig: MNIST DCT}) and CelebA (Fig.~\ref{fig: celeb DCT}), are very close for  models trained adversarially and using other schemes for the unperturbed case ($\epsilon = 0$), but the gap between them increases with increasing $\epsilon$'s, with adversarially trained models outperforming the other methods consistently. This behavior is qualitatively different from that for the Gaussian case (Fig.~\ref{fig:MNIST gauss} and Fig.~\ref{fig:Celeb gauss}), where the gap between adversarially trained networks and models trained using other (or no) regularizers is roughly constant for different $\epsilon$.
\subsection{Analysis with respect to Conditioning}
To check the conditioning, Fig.\ref{fig:svd gauss} shows the histogram for the singular values of the random Gaussian matrices. It can be seen that the condition number (ratio of maximum and minimum singular value) is close to $2$ which is very well conditioned 
for both data sets. On the other hand, the histogram of the same for the random DCT matrices (Fig.\ref{fig: svd DCT}) shows higher condition numbers -- $8.9$ for the $100\times784$ and $7.9$ for the $1000\times12288$ dimension matrices, which is ill-conditioned relative to the Gaussian ones.\\
Refering to the above analysis of conditioning and plots of the robustness measure $\hat{\rho}$ for the two types of matrices: random Gaussian vs.  random DCT indicate that the performance and behavior of the proposed min-max formulation depends on how well (or relatively ill)-conditioned the matrices are. This corroborates with the theoretical analysis for a simple reconstruction scheme (linear network) described in  Sec.~\ref{sec: theory_ana}.
\begin{figure*} [t]
    \centering
    \begin{subfigure}[b]{.32\linewidth}
    \centering
    \includegraphics[width=.99\textwidth]{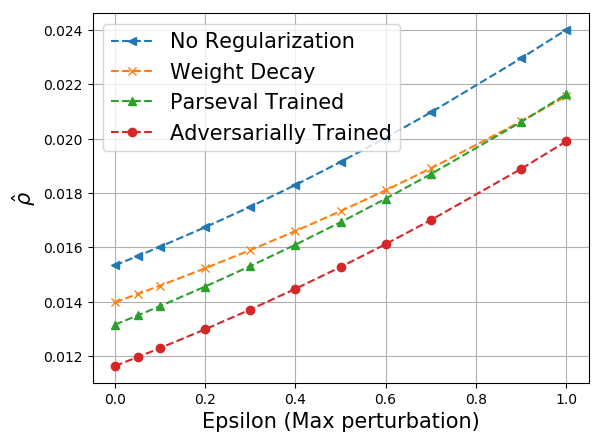}
    \caption{}\label{fig:MNIST gauss}
    \end{subfigure}
    \begin{subfigure}[b]{.345\linewidth}
    \centering
    \includegraphics[width=.99\textwidth]{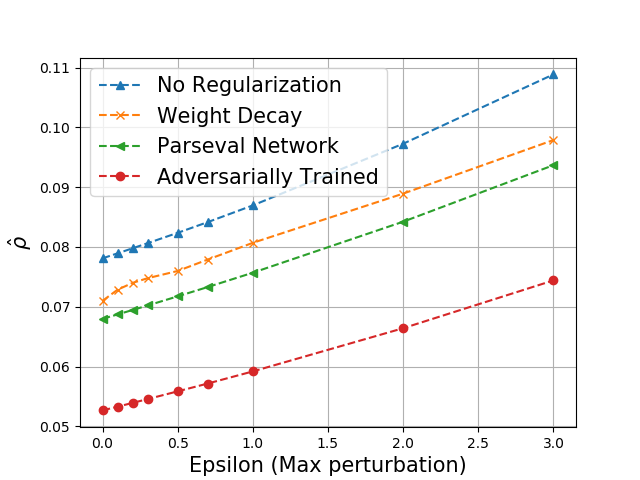}
    \caption{}\label{fig:Celeb gauss}
    \end{subfigure}
    \begin{subfigure}[b]{.32\linewidth}
    \centering
    \includegraphics[width=.99\textwidth]{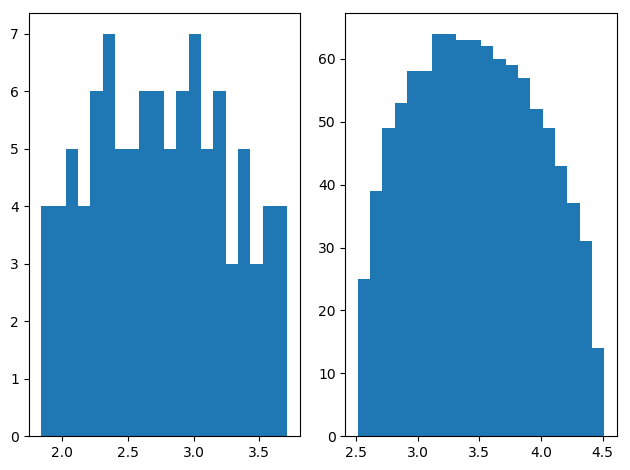}
    \caption{}\label{fig:svd gauss}
    \end{subfigure}\\
    \begin{subfigure}[b]{.32\linewidth}
    \centering
    \includegraphics[width=.99\textwidth]{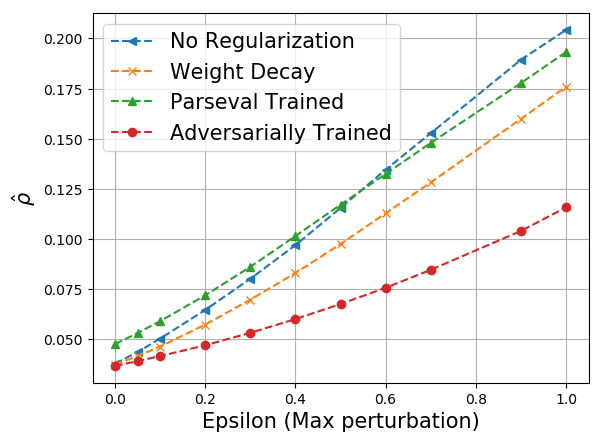}
    \caption{}\label{fig: MNIST DCT}
    \end{subfigure}
    \begin{subfigure}[b]{.345\linewidth}
    \centering
    \includegraphics[width=.99\textwidth]{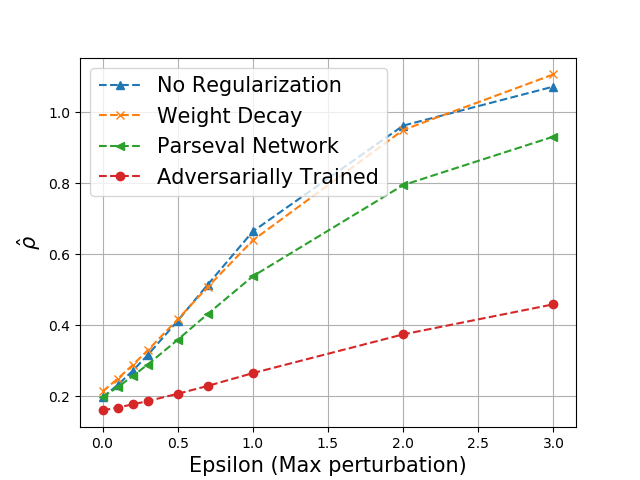}
    \caption{}\label{fig: celeb DCT}
    \end{subfigure}
    \begin{subfigure}[b]{.32\linewidth}
    \centering
    \includegraphics[width=.99\textwidth]{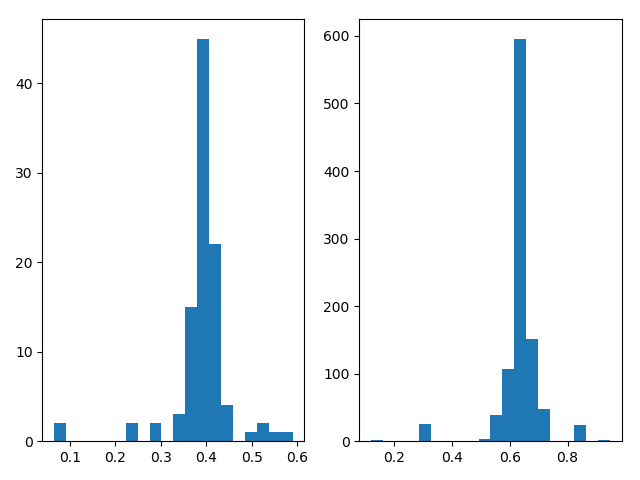}
    \caption{}\label{fig: svd DCT}
    \end{subfigure}\\
    \caption{\textbf{Row 1} corresponds to the random rows of Gaussian measurement matrix: (a) MNIST, (b) CelebA, (c) Distribution of the singular values for MNIST (left, $m = 100$) and CelebA (right, $m=1000$) cases. \textbf{Row 2} corresponds to random rows of the DCT measurement matrix: (a) MNIST, (b) CelebA, (c) Distribution of the singular values for MNIST (left, $m = 100$) and CelebA (right, $m=1000$) cases.
    }
    \label{fig:gauss dct}
\end{figure*}
\subsection{Linear Network for Reconstruction}
We perform an experiment using a linear reconstruction network in a simulated set-up to compare the theoretically obtained optimal robust reconstruction network with the one learned by our scheme  by optimizing the objective \eqref{eq:img_recon_2}. We take $50,000$ samples of a  signal $x \in \mathbb{R}^{20}$ drawn from $\mathcal{N}(0, I)$, hence, $\mathbb{E}(x) = 0 \text{ and } cov(x) = I$. For the measurement matrix $A \in  \mathbb{R}^{{10} \times 20}$, we follow the same strategy as in Sec.~\ref{sec:gauss_mat}, i.e. $A_{ij} \sim \mathcal{N}(0, 1/10)$. 
Since such matrices are well-conditioned, we replace $2$ singular values of $A$ by small values (one being $10^{-3}$ and another, $10^{-4}$) keeping other singular values and singular matrices fixed.  This makes the modified matrix $\tilde{A}$  ill-conditioned. We obtain the measurements $y = \tilde{A}x \in \mathbb{R}^{10}$. For reconstruction, we build a linear network $f$ having 1 fully-connected layer with no non-linearity i.e. $f = B \in \mathbb{R}^{20 \times 10}$. The reconstruction is given by $\hat{x} = \hat{B}y$, where $\hat{B}$ is obtained from:
\begin{equation}
    \argmin_B\max_{\delta:\|\delta\|_2 \leq \epsilon} \mathbb{E}_x \|B\tilde{A}x - x\|^2 + \lambda \|B(\tilde{A}x + \delta) - x\|^2
\end{equation}
We have used $\lambda=1$,  $\epsilon = 0.1$, learning rate $= 0.001$ and momentum term as $0.9$ in our experiments. We obtain the theoretically derived reconstruction $B$ using the result given in \eqref{eq:theor_res} (from theorem \ref{theorem}). To compare $B$ and $\hat{B}$, we examined the following three metrics:
\begin{itemize}
    \item 
    $\|\hat{B} - B\|_F/\|B\|_F = 0.024, \|\hat{B} - B\|_2/\|B\|_2 = 0.034$
    \item $\|I - B\tilde{A}\|_F / \|I - \hat{B}\tilde{A}\|_F = 0.99936$, where $I$ is the identity matrix of size $20\times20$
    \item $\kappa(B) = 19.231$, $\kappa(\hat{B}) = 19.311$, $\kappa$: condition number
\end{itemize}
The above three metrics indicate that $\hat{B}$ indeed converges to the theoretically obtained solution $B$.
\section{Conclusions}
In this work, we propose a min-max formulation to build a robust deep-learning-based image reconstruction models. To make this more tractable, we reformulate this using an auxiliary network to generate adversarial examples for which the image reconstruction network tries to minimize the reconstruction loss. We theoretically analyzed a simple linear network and found that using min-max formulation, it outputs singular-value(s) filter regularized solution which reduces the effect of adversarial examples for ill-conditioned matrices. Empirically, we found the linear network to converge to the same solution. Additionally, extensive experiments with non-linear deep networks for Compressive Sensing (CS) using random Gaussian and DCT measurement matrices on MNIST and CelebA datasets show that the proposed scheme outperforms other methods for different perturbations $\epsilon \geq 0$, however the behavior depends on the conditioning of matrices, as indicated by theory for the linear reconstruction scheme.

\appendix
\section{Appendix} \label{appen}
\textit{\textbf{Proof of Theorem 1}}:\newline
For the inverse problem of recovering the true $x$ from the measurement $y = Ax$, goal is to design a robust linear recovery model given by $\hat{x} = By = BAx$.

The min-max formulation to get robust model for a linear set-up:
\begin{flalign}
    \min_{B} \max_{\delta: \|\delta\|_2 \leq \epsilon} & \mathbb{E}_{x \in D} \|BAx - x\|^2 + \lambda \| B(Ax+\delta) - x\|^2 \nonumber \\
    \min_{B} \max_{\delta: \|\delta\|_2 \leq \epsilon} & \mathbb{E}_{x \in D}  (1 + \lambda)\|BAx - x\|^2 + \lambda \|B\delta\|^2  \nonumber \\ 
    & \quad \quad \quad + 2\lambda (B\delta)^T (BAx- x) 
\end{flalign}
Assuming, the dataset is normalized, i.e., $\mathbb{E}(x) = 0$ and $cov(x) = I$. The above optimization problem becomes:
\begin{flalign}
    \min_{B} \max_{\delta: \|\delta\|_2 \leq \epsilon} & \mathbb{E}_{x \in D} (1 + \lambda)\|(BA - I)x\|^2  + \lambda \|B\delta\|^2 \nonumber\\
    \min_{B} \max_{\delta: \|\delta\|_2 \leq \epsilon} & \mathbb{E}_{x \in D} (1 + \lambda)tr(BA-I)xx^T(BA-I)^T \nonumber\\
    & \quad \quad \quad + \lambda \|B\delta\|^2 
\end{flalign}
Since, $\mathbb{E}(tr(\cdot)) = tr(\mathbb{E}(\cdot))$, the above problem becomes:
\begin{flalign}
    \min_{B} & \max_{\delta: \|\delta\|_2 \leq \epsilon} (1 + \lambda) tr(BA-I)(BA-I)^T + \lambda \|B\delta\|^2 \nonumber \\
    & \min_{B} \max_{\delta: \|\delta\|_2 \leq \epsilon} (1 + \lambda) \|BA-I\|_F^2 + \lambda \|B\delta\|^2 
\end{flalign}
Using SVD decomposition of $A = USV^T$ and $B=MQP^T \implies M^TM = I, P^TP = I$ and $Q$ is diagonal. Assume that $\mathbb{G}$ defines the set satisfying the constraints of $M^TM = I, P^TP = I$ and $Q$ is diagonal.
\begin{flalign}
    \label{svd_form}
    \min_{M, Q, P \in \mathbb{G}} \max_{\delta: \|\delta\|_2 \leq \epsilon} & (1 + \lambda) \|MQP^TUSV^T-I\|_F^2 \nonumber \\
    & + \lambda \|MQP^T\delta\|^2
\end{flalign}
Since, only the second term is dependent on $\delta$, maximizing the second term with respect to $\delta$: \\
We have $\|MQP^T\delta\| = \|QP^T \delta\|^2$ since $M$ is unitary. Given, $Q$ is diagonal, $\|QP^T \delta\|^2$ w.r.t. $\delta$ can be maximized by having $P^T\delta$ vector having all zeros except the location corresponding to the $\max_{i} Q_i$. Since, $\|P^T\delta\| = \|\delta\|$, again because $P$ is unitary, so to maximize within the $\epsilon$-ball, we will have $P^T\delta = \epsilon[0, .., 0, 1, 0,..,0]$ where $1$ is at the $\arg max_{i} Q_i$ position. This makes the term to be:
\begin{align*}
    \max_{\delta: \|\delta\|_2 \leq \epsilon} \|MQP^T\delta\|^2 = \epsilon^2 (\max_i Q_i)^2
\end{align*}
Substituting the above term in equation \ref{svd_form}:
\begin{flalign}
  \min_{M, Q, P \in \mathbb{G}}  (1 + \lambda) & \|MQP^TUSV^T-I\|_F^2 + \lambda \epsilon^2 (\max_i Q_i)^2 \nonumber \\
    \min_{M, Q, P \in \mathbb{G} }  (1 + \lambda) & tr(MQP^TUSV^T-I) (MQP^TUSV^T-I)^T  \nonumber\\
   & \qquad + \lambda \epsilon^2 (\max_i Q_i)^2 \nonumber \\
    \min_{M, Q, P \in \mathbb{G}} (1 + \lambda) & tr(MQP^TUS^2U^TPQM^T  \nonumber \\
    & \qquad \hspace{-4em}- 2 MQP^TUSV^T + I) + \lambda \epsilon^2 (\max_i Q_i)^2 \nonumber \\
    \min_{M, Q, P \in \mathbb{G}} (1 + \lambda) & tr(P^TUS^2U^TPQ^2 - 2 MQP^TSV^T + I)\nonumber \\ 
    & \qquad + \lambda \epsilon^2 (\max_i Q_i)^2 \label{svd_trace_form}
\end{flalign}
For the above equation, only second term depends on $M$, minimizing the second term w.r.t. $M$ keeping others fixed:
\begin{align*}
    \min_{M: M^TM = I} tr(-2 MQP^TUSV^T) \\
\end{align*}
    Since, this is a linear program with the quadratic constraint, relaxing the constraint from $M^TM=I$ to $M^TM \leq I$ won't change the optimal point as the optimal point will always be at the boundary i.e. $M^TM=I$
\begin{flalign}
     \min_{M: M^TM \leq I} tr(-2 MQP^TUSV^T) \text{ which is a convex program} \nonumber
\end{flalign}
Introducing the Lagrange multiplier matrix $K$ for the constraint
\begin{flalign}
    \quad \mathcal{L}(M, K) & = tr(-2 MQP^TUSV^T + K(M^TM-I)) \nonumber
\end{flalign}
Substituting $G = QP^TUSV^T$ and using stationarity of Lagrangian
\begin{flalign}
        \Delta L_{M} & = M(K+K^T) - G^T = 0 \implies ML = G^T, L = K+K^T \nonumber
\end{flalign}
Primal feasibility: $M^TM \leq I$. Optimal point at boundary $\implies M^TM = I$.\\
Because of the problem is convex, the local minima is the global minima which satisfies the two conditions: Stationarity of Lagrangian ($ML = G^T$) and Primal feasibility ($M^TM=I$). By the choice of $M = V$, and $L = SU^TPQ$, both these conditions are satisfied implying $M = V$ is the optimal point. \\
Substituting $M = V$ in equation \ref{svd_trace_form}, we get:
\begin{flalign}
    \min_{Q, P \in \mathbb{G}} & (1 + \lambda)tr(P^TUS^2U^TPQ^2 - 2 VQP^TUSV^T + I) \nonumber \\
    & \qquad + \lambda \epsilon^2 (\max_i Q_i)^2 \nonumber\\
    \min_{Q, P \in \mathbb{G}} & (1 + \lambda)tr(P^TUS^2U^TPQ^2 - 2 QP^TUS + I) \nonumber \\ 
    & \qquad + \lambda \epsilon^2 (\max_i Q_i)^2 \nonumber \\
    \min_{Q, P \in \mathbb{G}} & (1 + \lambda) \|QP^TUS-I\|_F^2 + \lambda \epsilon^2 (\max_i Q_i)^2 \label{pq_form}
\end{flalign}
Denote the $i$-th column of $C=U^TP$ by $c_i$ and suppose that entries in $Q$ are in decreasing order and the largest entry $q_m$ in $Q$, has multiplicity $m$, the equation \ref{pq_form} becomes:
\begin{flalign}
   \min_{C, Q} & (1 + \lambda) \sum_{i=1}^m \|q_mSc_i -e_i\|^2 +  \lambda \epsilon^2 q_m^2 \nonumber \\ 
   & + (1 + \lambda) \sum_{i=m+1}^n \|qiSc_i -e_i\|^2    \label{quad_q_form}
\end{flalign}
If we consider the last term i.e. $i> m$, it can be minimized by setting $c_i = e_i$ which is equivalent to choose $P_i = U_i$ and $q_i = 1/S_i$. This makes the last term ($=0$), using $h = \lambda \epsilon^2/(1+\lambda)$, making the equation \ref{quad_q_form} as:
\begin{align*}
    \min_{C, Q} \sum_{i=1}^m (c_i^TSq_m^2 Sc_i - 2 e_i^Tq_mSc_i + e_i^Te_i) + h q_m^2 \\
    \min_{C, Q} q_m^2(\sum_{i=1}^m c_i^TS^2c_i  +  h) - 2q_m \sum_{i=1}^m S_i C_{ii} + \sum_{i=1}^m e_i^Te_i
\end{align*}
The above term is upward quadratic in $q_m$, minima w.r.t. $q_m$ will occur at $q_m^* = \frac{\sum_{i=1}^m S_i C_{ii}}{(\sum_{i=1}^m c_i^TS^2c_i  +  h)}$, which will make the quadratic term as $\sum_{i=1}^m e_i^Te_i  - \frac{(\sum_{i=1}^m S_i C_{ii})^2}{ (\sum_{i=1}^m c_i^TS^2c_i  +  h)}$, which has to be minimized w.r.t $C$ 
\begin{flalign}
    \min_{C} \sum_{i=1}^m e_i^Te_i  - \frac{(\sum_{i=1}^m S_i C_{ii})^2}{ (\sum_{i=1}^m c_i^TS^2c_i  +  h)} \nonumber\\
    \max_{C} \frac{(\sum_{i=1}^m S_i C_{ii})^2}{ (\sum_{i=1}^m c_i^TS^2c_i  +  h)} \nonumber\\
    \max_{C} \frac{(\sum_{i=1}^m S_i C_{ii})^2}{ \sum_{i=1}^m S_i^2 C_{ii}^2 + \sum_{j \neq i} S_j^2 C_{ij}^2 + h} \label{diag_nondiag}
\end{flalign}
Since $C = U^TP \implies C_{ij} = u_i^Tp_j \implies \|C_{ij}\| \leq 1$. To maximize the term given by the equation \ref{diag_nondiag}, we can minimize the denominator by setting the term $C_{ij} = 0$, which makes the matrix $C$ as diagonal. \\
Divide the matrix $U$ and $P$ into two parts: one corresponding to $i \leq m$ and another $i > m$, where $i$ represents the column-index of $C = U^TP$. \\
Let $U = [U_1 | U_2]$ and $P = [P_1 | P_2]$. From above, we have $P_2 = U_2$ for $i > m$, making $P = [P_1 | U_2]$. \\
\begin{align*}
    U^T = \begin{bmatrix}
U_1^T\\ \hline
U_2^T 
\end{bmatrix}
\text{ and } P = [P_1 | U_2] \\
U^TP = \begin{bmatrix}
U_1^TP_1 & U_1^TU_2\\ 
U_2^TP_1 & U_2^TU_2 
\end{bmatrix}
= \begin{bmatrix}
U_1^TP_1 & \mathbf{0}\\ 
U_2^TP_1 & I
\end{bmatrix}
\end{align*}
Since, $U^TP$ is diagonal, we have $U_2^TP_1 = \mathbf{0}$, $U_1^TP_1 = \Gamma$ where $\Gamma$ is diagonal. Also, we have $P_1^TP_1 = I$. Only way to satisfy this would be making $P_1 = U_1$ which makes $P = U$ and $C = I$. It also results in \\
\begin{equation}
    q_m^* = \frac{\sum_{i=1}^m S_i}{ \sum_{i=1}^m S_i^2 + h}
\end{equation}

Hence, the resulting $B$ would be of the form $MQP^T$ where:
\begin{align}
    M = V, P = U \hspace{4em} \nonumber \\
    Q = \begin{bmatrix}
    q_m^* & 0 & ... & 0 \\
    0  & q_m^* & .. & 0\\
    :  &  :  & : & :\\
    :  &  :  & : & :\\
    0 &  ... & 1/S_{m+1} & ..\\
    :  & :   &  : & :\\
    :  &  :  & : & :\\
    0  & ... & 0 & 1/S_{n}
    \end{bmatrix}
\end{align}

\bibliography{example_paper}
\bibliographystyle{icml2020}



\end{document}